\documentclass{article}

\PassOptionsToPackage{round}{natbib}

\usepackage[final]{neurips_2019}
\usepackage[utf8]{inputenc} \usepackage[T1]{fontenc}    \usepackage{hyperref}       \usepackage{url}            \usepackage{booktabs}       \usepackage{amsmath, amssymb, amsthm, amsfonts, bbm}
\usepackage{nicefrac}       \usepackage{microtype}      \usepackage{comment}
\usepackage{color}
\usepackage{caption}
\usepackage{cleveref}
\usepackage{algorithm}
\usepackage{algorithmic}
\usepackage{graphicx}    
\usepackage{mathtools}
\usepackage{nicefrac}
\usepackage{xcolor}
\usepackage{caption}
\usepackage{subcaption}
\usepackage{mathrsfs} \usepackage{wrapfig}

\graphicspath{{figures/}}

\newcommand{\C}{\mathbb{C}}
\newcommand{\E}{\mathbb{E}}
\newcommand{\R}{\mathbb{R}}
\newcommand{\N}{\mathbb{N}}
\newcommand{\T}{\mathbb{T}}

\def\mydefc#1{\expandafter\def\csname c#1\endcsname{\mathcal{#1}}}
\def\mydefallc#1{\ifx#1\mydefallc\else\mydefc#1\expandafter\mydefallc\fi}
\mydefallc ABCDEFGHIJKLMNOPQRSTUVWXYZ\mydefallc

\DeclareMathOperator*{\maximize}{maximize\ }
\DeclareMathOperator{\Var}{Var}
\DeclareMathOperator{\tr}{\mathbf{tr}}

\DeclareMathOperator{\Uniform}{Uniform}
\DeclarePairedDelimiter{\ceil}{\lceil}{\rceil}

\newcommand{\Norm}{\mathcal{N}}

\newcommand{\PSD}[1]{\mathbb{P}_+^{#1}}
\newcommand{\PD}[1]{\mathbb{P}_{++}^{#1}}
\newcommand{\norm}[1]{\left\| {#1} \right\|}

\newcommand{\term}[1]{\textbf{{Term}}_{#1}}

\newcommand{\onevars}[1]{\mathbbm{1}(#1)}
\newcommand{\sumH}{\sum_{t=1}^H}
\newcommand{\sumIJ}{\sum_{1 \leq i, j \leq H}}
\newcommand{\ABKradius}{f}
\newcommand{\gtg}{{\hat g^\top \hat g}}
\newcommand{\half}{{\frac{1}{2}}}

\newcommand{\initstate}{\varrho}
\newcommand{\sumvar}{\nu}
\newcommand{\rtau}{\phi}

\newtheorem{theorem}{Theorem}
\newtheorem{lemma}[theorem]{Lemma}

\title{Analyzing the Variance of Policy Gradient Estimators\\for the Linear-Quadratic Regulator}
\date{\today}
\author{
\hspace{-2mm}
James A. Preiss\thanks{Equal contribution.}, \ 
S\'ebastien M. R. Arnold\footnotemark[1], \ 
Chen-Yu Wei\footnotemark[1] \\
Department of Computer Science \\
University of Southern California \\
Los Angeles, USA \\
\texttt{\{japreiss, seb.arnold, chenyu.wei\}@usc.edu}
\And
Marius Kloft \\
Department of Computer Science \\
T.U. Kaiserslautern \\
Kaiserslautern, Germany \\
\texttt{neu@cs.uni-kl.de}
\hspace{-2mm}
}

\begin{document}

\maketitle

\begin{abstract}
We study the variance of the REINFORCE policy gradient estimator in environments with continuous state and action spaces, linear dynamics, quadratic cost, and Gaussian noise. These simple environments allow us to derive bounds on the estimator variance in terms of the environment and noise parameters. We compare the predictions of our bounds to the empirical variance in simulation experiments.
\end{abstract}

\section{Introduction}\label{introduction} 
Policy gradient (PG) algorithms are widely used for reinforcement learning (RL)
in continuous spaces.
PG methods construct an unbiased estimate of the gradient of the RL objective with respect to the policy parameters.
They do so without the complication of intermediate steps of dynamics modeling or value function approximation \citep{sutton-barto}.
However, the gradient estimate is known to suffer from high variance. This makes PG methods sample-inefficient with respect to environment interaction,
creating an obstacle for applications to real physical systems.

In this paper, we seek a more detailed understanding of how the PG gradient estimate variance
relates to properties of the continuous-space Markov decision process (MDP)
that defines the RL problem instance.
Such characterization is well-developed for discrete state spaces \citep[e.g.][]{greensmith2004variance},
but in continuous spaces, a detailed breakdown is not possible
without further restrictions on the set of MDPs and the policy class.
We choose to examine systems with
linear dynamics, linear policy, quadratic cost, and Gaussian noise,
known as LQR systems in control theory.

LQR systems are a popular case study for analyzing RL algorithms in continuous spaces.
\citet{fazel2018global} show that the optimization landscape is neither convex nor smooth,
but still admits global convergence and PAC-type complexity bounds for
a zeroth-order optimization method that explores in parameter space. \citet{malik-derivativefree-lqr} establish tighter error bounds with more detailed problem-dependence for the same class of algorithms,
extended to include noisy dynamics.
\citet{yang-actorcritic-lqr} show convergence for an actor-critic method.
\citet{recht-tour-rl-control} provide a broad overview including
summaries of the group's prior work on model-based methods and value function approximation.
Our variance bounds are similar to those of \citet{malik-derivativefree-lqr}, but we 
focus on characterizing the variance of policy gradient methods,
which explore by taking random actions at each time step,
rather than parameter-space exploration methods.
\citet{tu-modelbased-lqr} show related bounds for a restricted class of LQR systems
as an intermediate step towards sample complexity results.

The earliest and simplest policy gradient algorithm is REINFORCE \citep{williams-REINFORCE}.
More recent algorithms such as TRPO and PPO \citep{schulman-trpo,schulman-ppo}
extend the basic idea of REINFORCE with techniques to
inhibit the possibility of making very large changes in the policy action distribution in a single step.
In a benchmark test \citep{duan-benchmarking}, these algorithms generally learned better policies
than REINFORCE, but their additional complexity makes them hard to analyze.

Our primary contributions are derivations of bounds on the variance of the REINFORCE gradient estimate
as an explicit function of the dynamics, reward, and noise parameters of the LQR problem instance.
We validate our bounds with comparisons to the empirical gradient variance
in random problems.
We also explore the relationship between gradient variance and sample complexity,
but find it to be less straightforward,
as the problem parameters that affect variance also affect the optimization landscape.
We emphasize that our goal is not to draw a conclusion about the utility of using REINFORCE to solve LQR problems,
but rather to use LQR as an example system that is simple enough to allow us to ``look inside'' the REINFORCE policy gradient estimator.

\section{Problem setting}
In this section, we define
the general finite-horizon reinforcement learning problem,
the REINFORCE policy gradient estimator,
and the LQR optimal control problem.
We use the notation $\cP(\cX)$
for the set of probability distributions over a measurable set $\cX$.
For arbitrary sets $\cX$ and $\cY$, the set of all functions $\cX \mapsto \cY$ is denoted as $\cY^\cX$.
Let $\|\cdot\|$ denote the $\ell_2 - \ell_2$ operator norm of a matrix or the $\ell_2$ norm of a vector.
The spectral radius (magnitude of largest eigenvalue) of a square matrix is denoted by $\rho(\cdot)$.
The the set of positive semidefinite (resp. positive definite) $k \times k$ matrices is denoted by $\PSD{k}$ (resp. $\PD{k}$).
Finally, $\Sigma^{\half}$ denotes the principal matrix square root of $\Sigma \in \PSD{k}$.

\paragraph{RL problem statement.}
\label{sec:rl-statement}
Reinforcement learning takes place in a Markov Decision Process (MDP)
defined by
state space $\cS$,
action space $\cA$,
initial state distribution $\initstate \in \cP(\cS)$,
state transition function $T : \cS \times \cA \mapsto \cP(\cS)$,
and reward function $r : \cS \times \cA \mapsto \cP(\R)$.
The agent's actions are drawn according to a stochastic policy $\pi : \cS \mapsto \cP(\cA)$.
Let $\tau$ denote a state-action trajectory
$s_1, a_1, s_2, a_2, \dots, s_H, a_H$,
of horizon $H \in \N$, and
${\T = (\cS \times \cA)^H}$
the set of all such trajectories.
The policy $\pi$ induces a trajectory distribution
$p_\pi \in \cP(\T)$, meaning
${s_1 \sim \initstate},\
{a_t \sim \pi(s_t)},\
{s_{t+1} \sim T(s_t,a_t)},\
{r_t \sim r(s_t, a_t)}$.
The RL optimization problem is
defined over a tractable policy space $\Pi \subseteq \cP(\cA)^{\cS}$ as follows:
\begin{equation}\begin{split}
\maximize_{\pi \in \Pi} \; &
J(\pi) = \E_{\tau \sim p_\pi} [\rtau(\tau)],
\; \text{where} \;
\rtau(\tau) = \textstyle \sum_{t=0}^H r_t.
\label{eq:rl-objective}
\end{split}\end{equation}
The MDP parameters $(\initstate, T, r)$ are unknown to the RL algorithm.
The algorithm must learn exclusively by sampling from $p_\pi$.

\paragraph{Policy gradient algorithms.}
\label{sec:policy-gradient}
Suppose $\pi$ is parameterized by a real-valued vector $\theta$.
Analytical gradient descent of $\nabla_\theta J(\pi)$ is not possible
because $T$ and $r$ are only accessible via sampling, with unknown gradients.
There exist generic derivative-free optimization algorithms that solve
such problems via perturbations in parameter space,
but in RL problems, it is also possible to explore via stochastic actions instead.
The simplest policy gradient algorithm is REINFORCE \citep{williams-REINFORCE},
which relies on the following identity (assuming sufficient regularity conditions):
\begin{equation}
	\nabla_\theta J(\pi)
	= \nabla_\theta \int_\T p_\pi(\tau) \rtau(\tau) d\tau\\
	= \E_{\tau \sim p_\pi}\left[ \rtau(\tau)\sum_{t=1}^H \nabla_\theta \log \pi(a_t|s_t) \right]. 
	\label{eq:reinforce}
\end{equation}
An unbiased estimate of the latter expectation can be computed by executing $\pi$ in the MDP for one full trajectory.
Unfortunately, this estimate is known to have high variance. Variance reduction is possible by exploiting the Markov property (past rewards are independent of future actions)
and/or using control variates~\citep{greensmith2004variance},
but we analyze plain REINFORCE here for simplicity.

\paragraph{LQR systems.}
\label{sec:lqr}
A discrete-time stochastic \emph{linear-quadratic regulator} (LQR) system
with state space $\cS = \R^n$ and action space $\cA = \R^m$
is defined by linear dynamics with additive Gaussian noise:
\begin{equation}
    s_{t+1} = A s_t + B a_t + \epsilon^s_t,\ \quad \epsilon^s_t \sim \Norm(\mathbf{0}, \Sigma_s),
\label{eq:lqr-dynamics}
\end{equation}
for dynamics matrices $A \in \R^{n \times n},\ B \in \R^{n \times m}$,
and noise covariance $\Sigma_s \in \PSD{n}$.
The reward function is given by $r_t = -(s_t^T Q s_t + a_t^T R a_t)$
for cost matrices $Q \in \PSD{n},\ R \in \PD{m}$.
Intuitively, the goal is to drive the state towards zero without using too much control effort.
The initial state $s_1$ follows an arbitrary zero-mean Gaussian distribution.
A well-known result in control theory \citep{kwakernaak1972linear}
states that, if the system $(A, B)$ is controllable,
the infinite-horizon objective
\begin{equation}
\label{eq:lqr-cost}
\lim_{H \to \infty} \E_{s_1, \{\epsilon^s_t\}_{t = 1}^H} \left[ \textstyle \sum_{t=1}^H r_t \right]
\end{equation} 
is maximized by a stationary linear policy $a_t = K^\star s_t,\ K^\star \in \R^{m \times n}$.
The value of  $K^\star$ depends on $(A, B, Q, R)$,
but not on the distributions of $\Sigma_s$ or $s_1$.
The same $K^\star$ is also the optimal controller for the deterministic version of the problem.
$K^\star$ can be computed efficiently \citep{vandooren-riccati}. To apply REINFORCE, the policy must be stochastic,
so we consider linear \textit{stochastic} policies
\begin{equation}
a_t = Ks_t + \epsilon_t^a, \quad \epsilon_t^a \sim \Norm(\mathbf{0}, \Sigma_a),
\label{eq:lqr-policy}
\end{equation}
for $K \in \R^{m \times n}, \Sigma_a \in \PD{m}$.
The state noise $\Sigma_s$ is an immutable property of the system,
but the action noise $\Sigma_a$ is not.
Instead, it is usually chosen by the user of the RL algorithm,
or learned as a parameter using \eqref{eq:reinforce}.
Genuine noise in the system actuators can be subsumed into $\Sigma_s' = \Sigma_s + B \Sigma_a$.

In the RL literature, action noise is usually seen as either
1) a tool for exploring of the state space,
2) a method of regularization to avoid converging on bad local optima,
or 3) a consequence of a probabilistic interpretation of the RL problem \citep{levine-rl-inference}.
Its effect on the RL optimization algorithm is less frequently discussed,
but in this work we find that it can be significant.

\section{Main result: Variance bounds on the REINFORCE estimator}
In this section, we present bounds on the variance of the REINFORCE estimator for LQR systems.
The instantiation of REINFORCE (\cref{eq:reinforce}) for the system (\cref{eq:lqr-dynamics,eq:lqr-cost,eq:lqr-policy}) using a single trajectory is:
\begin{align}
    \hat{g} = \left(\sum_{t=1}^H \Sigma_a^{-1} \epsilon_t^{a}s_t^\top  \right)\left( \sum_{t=1}^H r_t \right) \in \R^{m\times n}.
    \label{eqn:g-hat}
\end{align}
The estimate $\hat g$ is a function of the independent random variables
$\{ \epsilon^a_t, \epsilon^s_t \}_{t=1}^H$.
Although $s_t$ is linear in 
$\{ \epsilon^a_\tau, \epsilon^s_\tau \}_{\tau=1}^{t-1}$,
$r_t$ is quadratic in $s_t$, so the overall form of $\hat g$ is
a product of a sum and a sum of products of sums.
Therefore, while it is possible to apply matrix concentration inequalities
\citep{tropp-matrix-concentration}
to bound $\|s_t - \E[s_t]\|$ with high probability,
it is more difficult to bound the dispersion of $\hat g$.
Instead, we use a more specialized method to derive a bound on
\[
\textstyle
\sumvar(\hat g) \triangleq
\sum_{i=1}^m \sum_{j=1}^n \Var(\hat g_{i,j}) =
\E\left[\tr(\hat{g}^\top\hat{g})\right] -
\tr(\E[\hat{g}]^\top\E[\hat{g}]),
\]
which we simplify by bounding
$\E\left[\tr(\hat{g}^\top\hat{g})\right]$.

\begin{theorem}
    \label{thm:upper}
    If $\rho(A+BK)<1$, then
    \(
        \E\left[\tr(\hat{g}^\top\hat{g})\right]\leq O\left(\bar n^4 C_1^2 C_2^2\right),
    \)
    where
    \begin{align*}
        C_1&=
        \mu^2 \norm{\Sigma_a^{-\frac{1}{2}}}  \left(\norm{s_1} + \sigma H\right)H^{\prime}, \\
        C_2&= 
        \norm{R} \norm{\Sigma_a}H +  \mu^2 \left(\norm{Q} + \norm{R} \norm{K}^2\right)\left(\norm{s_1}^2 + \sigma^2 H\right)H^{\prime 2},
    \end{align*}
                        $\bar n \triangleq \max \{n, m\}$,
    $\sigma \triangleq \norm{\Sigma_s^{\frac{1}{2}}} + \norm{B \Sigma_a^{\frac{1}{2}}}$,
    $H'\triangleq \min\left\{H ,\frac{1}{1-\rho(A+BK)}\right\}$,
    and $\mu$ is a constant bounding the transient behavior of $\|A+BK\|^t$,
    with more details provided in \Cref{appendix:upper}.
\end{theorem}
\emph{Proof Sketch.}
\begin{enumerate}
\item Rewrite $\epsilon^a_t, \epsilon^s_t$ as $\Sigma_a^{\half} \delta^a_t, \Sigma_s^{\half} \delta^s_t$,
where the $\delta^a_t, \delta^s_t$  are unit-Gaussian random variables.
\item Bound $\tr(\gtg)$ by $P$, a polynomial function of
the $\chi$-distributed random variables $\{ \|\delta^a_t\|, \|\delta^s_t\| \}_{t=1}^H$
with nonnegative coefficients.
\item Bound the sum of the coefficients of $P$ by substituting $1$ for all $\chi$ random variables.
\item Bound for the expectation of each monomial in $P$ using the moments of the $\chi$ distribution.
\end{enumerate}
A detailed proof of \Cref{thm:upper} is given in \Cref{appendix:upper}.

For a special case of scalar states and actions, we also show a lower bound on $\E [\hat g^2]$.
Since $\E[\hat g] = 0$ at a local optimum, this lower bound corresponds to
the variance caused strictly by noise in the system when the policy is already optimal.
Here, the matrices $A, B, K, Q, R$ are denoted as $a, b, k, q, r$,
and $\sigma_s,\ \sigma_a$ denote the standard deviation (not variance)
of state and action noise.
This notation $r$ is different from the notation $r_t$ for reward.
\begin{theorem}
    \label{thm:lower}
    If $m = n = 1$ and $0 \leq a + bk < 1$, then
    \(
        \E [\hat g^2] \geq \Omega(c_1^2 c_2^2), 
    \)
    with
    \begin{align*}
        c_1 &= \frac{1}{\sigma_a}\left(|s_1| + \sigma \sqrt{H}\right)\sqrt{h'}, \\
        c_2 &= r\sigma_a^2 H + (q+rk^2)(s_1^2+\sigma^2 H)h',
    \end{align*}
    where $\sigma \triangleq \sigma_s + b \sigma_a$ and $h' \triangleq \min \left\{ H, \frac{1}{1 - (a + bk)^2} \right\}$.  
\end{theorem}
A detailed proof of \Cref{thm:lower} is given in \Cref{appendix:lower}.
If we reduce  the upper bound of \Cref{thm:upper} to its scalar case,
all terms match with the notable exception of the horizon-related terms
$H$ and $H'$ (compare to $h'$),
which appear squared in several places in \Cref{thm:upper}
compared to the equivalent term in \Cref{thm:lower}.
There is another gap in the denominators of $H'$ and $h'$
since $\frac{1}{1-x^2} < \frac{1}{1-x}$ on the domain $x \in (0, 1)$.
We conjecture our upper bound can be tightened by fully exploiting the independence of the noise variables $\epsilon_s,\ \epsilon_a$.

\section{Experiments}
\label{sec:experiments}
\begin{figure}[t]
\begin{subfigure}{0.32\textwidth}
\centering
\includegraphics[width=\columnwidth]{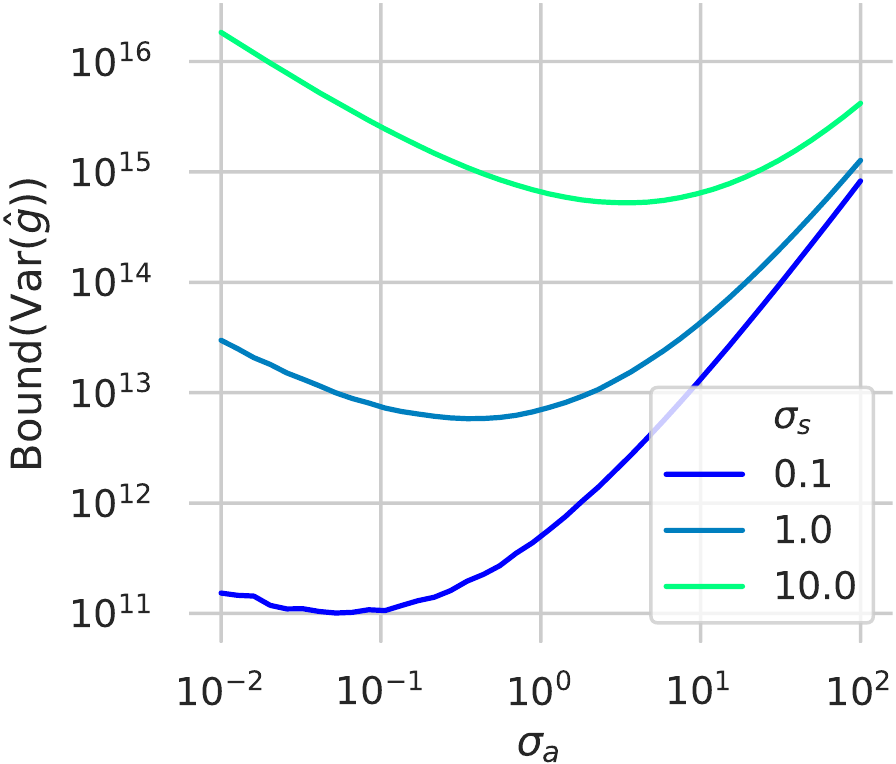} \\
\vspace{3mm}
\includegraphics[width=\columnwidth]{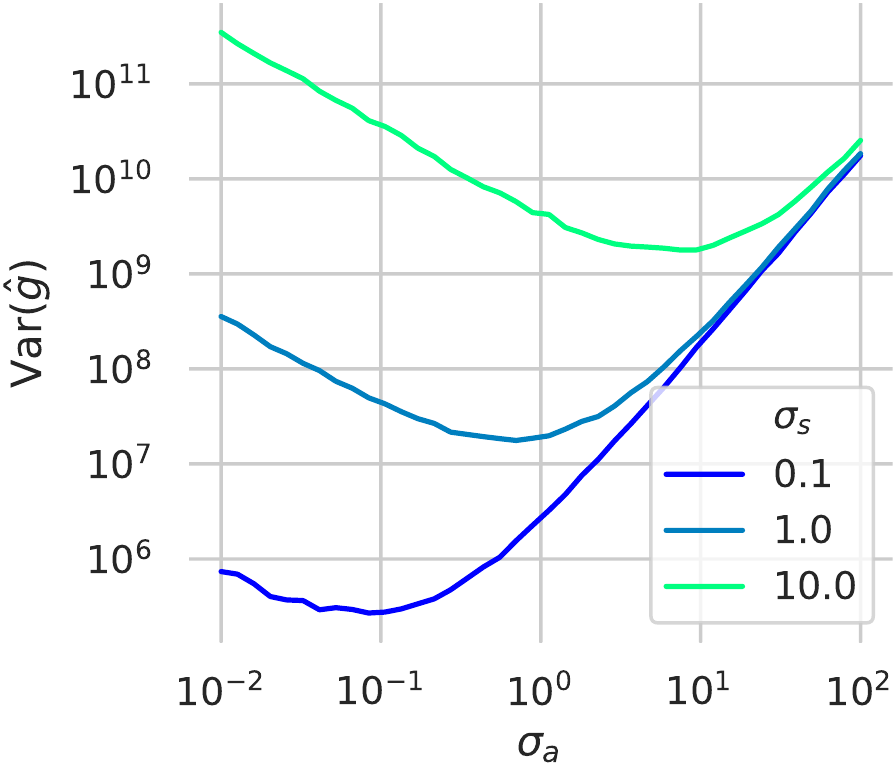} \\
\caption{Relationship between $\Sigma_s$ and $\Sigma_a$}
\label{fig:sigmaA}
\end{subfigure}
\hfill
\begin{subfigure}{0.32\textwidth}
\centering
\includegraphics[width=\columnwidth]{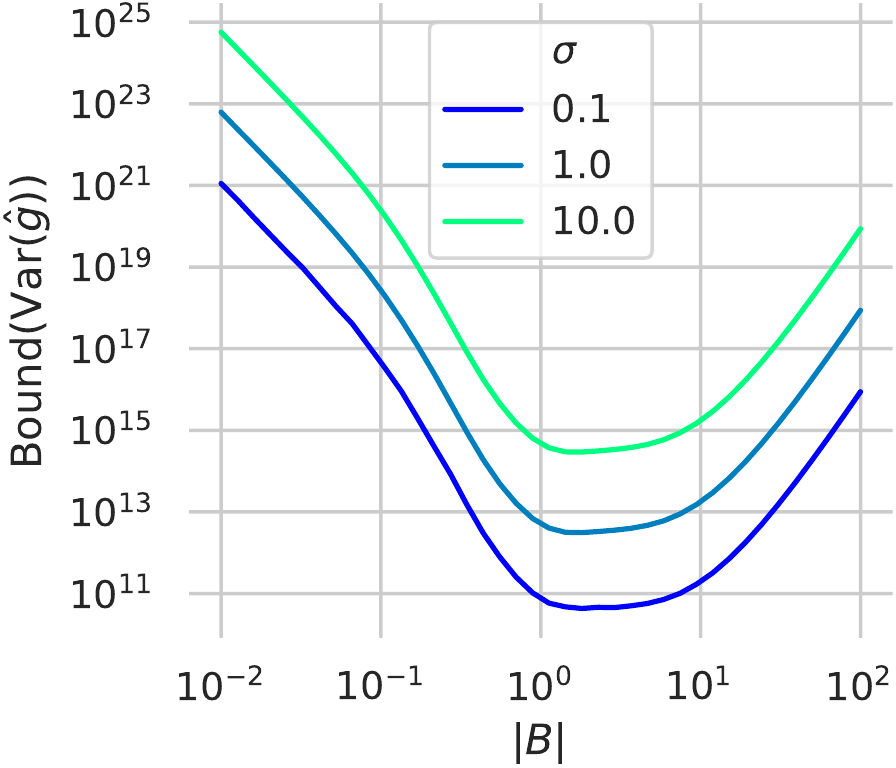} \\
\vspace{3mm}
\includegraphics[width=\columnwidth]{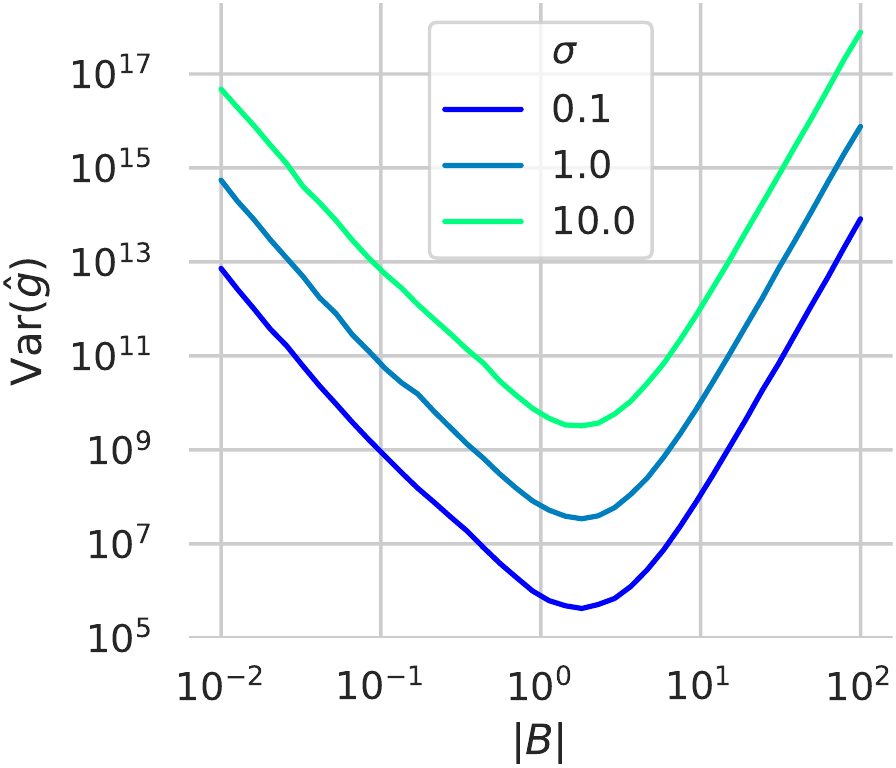} \\
\caption{Control authority $|B|$}
\label{fig:Bmag}
\end{subfigure}
\hfill
\begin{subfigure}{0.32\textwidth}
\centering
\includegraphics[width=\columnwidth]{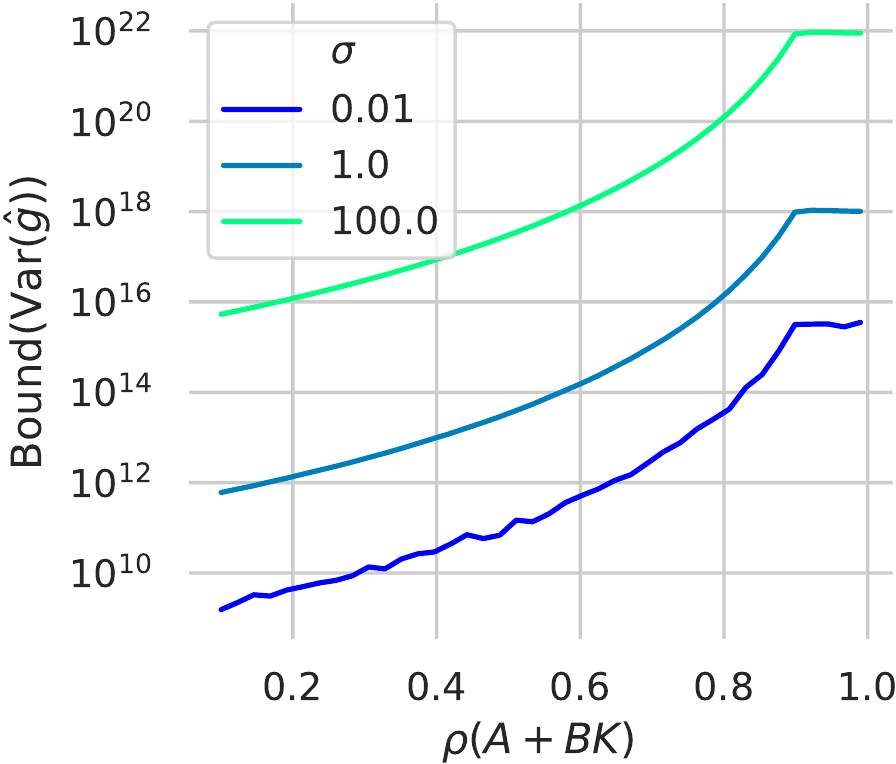} \\
\vspace{3mm}
\includegraphics[width=\columnwidth]{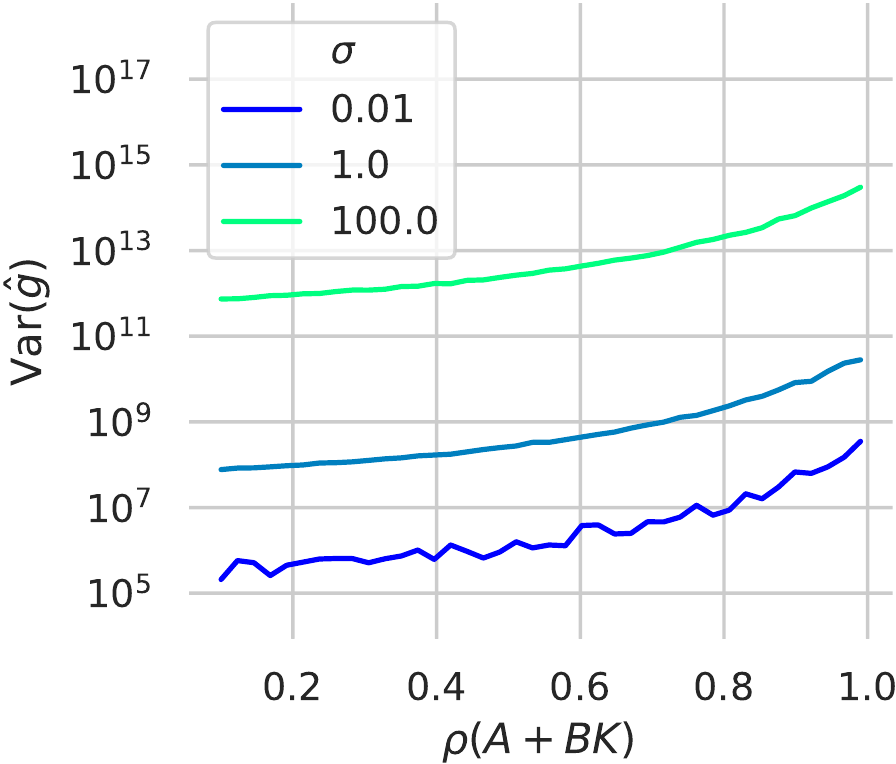} \\
\caption{Stability $\rho(A + BK)$}
\label{fig:rho}
\end{subfigure}
\caption{
Comparison between our upper bound from \Cref{thm:upper} (top)
and the empirically measured variance (bottom)
as they relate to various parameters of the LQR problem.
Behavior is qualitatively similar for action noise covariance (a)
and control authority (b),
but less similar for the stability (c)
where our bounds are loose.
Further discussion is in \Cref{sec:sigmaA,sec:Bmag,sec:rho}.
}
\label{fig:compare_empirical}
\end{figure}

In the first set of experiments, we compare our upper bound of $\sumvar(\hat g)$
to its empirical value when executing REINFORCE in randomly generated LQR problems.
Our results show qualitative similarity in the parameters
for which our upper and lower bounds match.
On the other hand, the gap with respect to stability-related parameters is also visible.
For each experiment shown here, we repeated the experiment with different random seeds
and observed qualitatively identical results.

We generate random LQR problems with the following procedure.
We sample each entry in $A$ and $B$ i.i.d. from $\Norm(0, \sigma{=}n^{-1/2})$
and $\Norm(0, \sigma{=}m^{-1/2})$ respectively.
To construct a random $k \times k$ positive definite matrix,
we sample from the $\mathrm{Wishart}(k^{-1}I, k)$ distribution by
computing $Q = X^T X$ for $X$ i.i.d. analogous to $A$.
The scale factor $k^{-1}$ ensures that if the vector $x$
is distributed by $\Norm(0, k^{-1}I)$,
such that $\E[\|x\|^2] = 1$,
then $\E[x^T Q x] = 1$.
We sample $Q, R, \Sigma_s$, and $\Sigma_a$ this way.

In each experiment, we vary some of these parameters systematically
while holding the others constant,
allowing us to visualize the impact of each parameter on $\sumvar(\hat g)$.
We plot the upper bound of \Cref{thm:upper}
on the top row of \Cref{fig:compare_empirical},
and the empirical estimate of $\sumvar(\hat g)$ on the bottom row.
In both cases, since the variance depends on the initial state $s_1$,
we sample $N = 100$ initial states $s_1$ from  $\Norm(0, n^{-1} I)$
and plot $\E_{s_1 \sim \Norm(0, n^{-1} I)} \sumvar(\hat g)$.
We estimate $\sumvar(\hat g) |_{s_1 = s}$ for a particular initial state $s$ by sampling $30$ trajectories with random
$\epsilon^a,\ \epsilon^s$.

\paragraph{Effect of \texorpdfstring{$\Sigma_a$}{Sigma_a}.}
\label{sec:sigmaA}
In this experiment, we generate a random LQR problem
and replace $\Sigma_a$ with $\sigma_a I$
for $\sigma_a$ geometrically spaced in the range $[10^{-2}, 10^2]$.
Using a scaled identity matrix is common practice when applying RL to a problem where there is
no \emph{a priori} reason to correlate the noise between
different action dimensions.
We evaluate the variance at the value $K = K^\star$,
where $K^\star$ is the infinite-horizon optimal controller computed using traditional LQR synthesis,
as described in \Cref{sec:lqr}.
Using $K^\star$ ensures that $\rho(A + BK) < 1$,
a required condition to apply \Cref{thm:upper}.

Results are shown in \Cref{fig:sigmaA}.
The separate line plots correspond to scaling the random $\Sigma_s$ by the values $\{ 0.1, 1, 10 \}$,
while the $x-$axis corresponds to the value of $\sigma_a$.
For each value of $\sigma_s$,
there appears to be a unique $\sigma_a$ that minimizes $\sumvar(\hat g)$,
and this value of $\sigma_a$ increases with $\sigma_s$.
This phenomenon appears in both the bound and empirical variance.

\paragraph{Effect of \texorpdfstring{$\|B\|$}{|B|}.}
\label{sec:Bmag}
In this experiment, we generate a random problem where $m = n$
and replace $B$ with $b I$
for $b$ geometrically spaced in the range $[10^{-2}, 10^2]$.
The resulting system essentially gives the policy direct control over each state.
For each $B$, we compute a separate infinite-horizon optimal $K$ and sample the variance for different $s_1$.
Results are shown in \Cref{fig:Bmag}.
The separate line plots correspond to scaling both the random $\Sigma_s$
and the random $\Sigma_a$
by the values $\{ 0.1, 1, 10 \}$.
The $x-$axis corresponds to the value of $b$.
Again, there appears to be a unique $\|B\|$ that minimizes $\sumvar(\hat g)$,
but its value changes minimally for different magnitudes of $\Sigma_s,\ \Sigma_a$.

\paragraph{Effect of \texorpdfstring{$\rho(A + BK)$}{Spectral Radius of A + BK}.}
\begin{wrapfigure}{r}{50mm}
  \begin{center}
    \includegraphics[width=50mm]{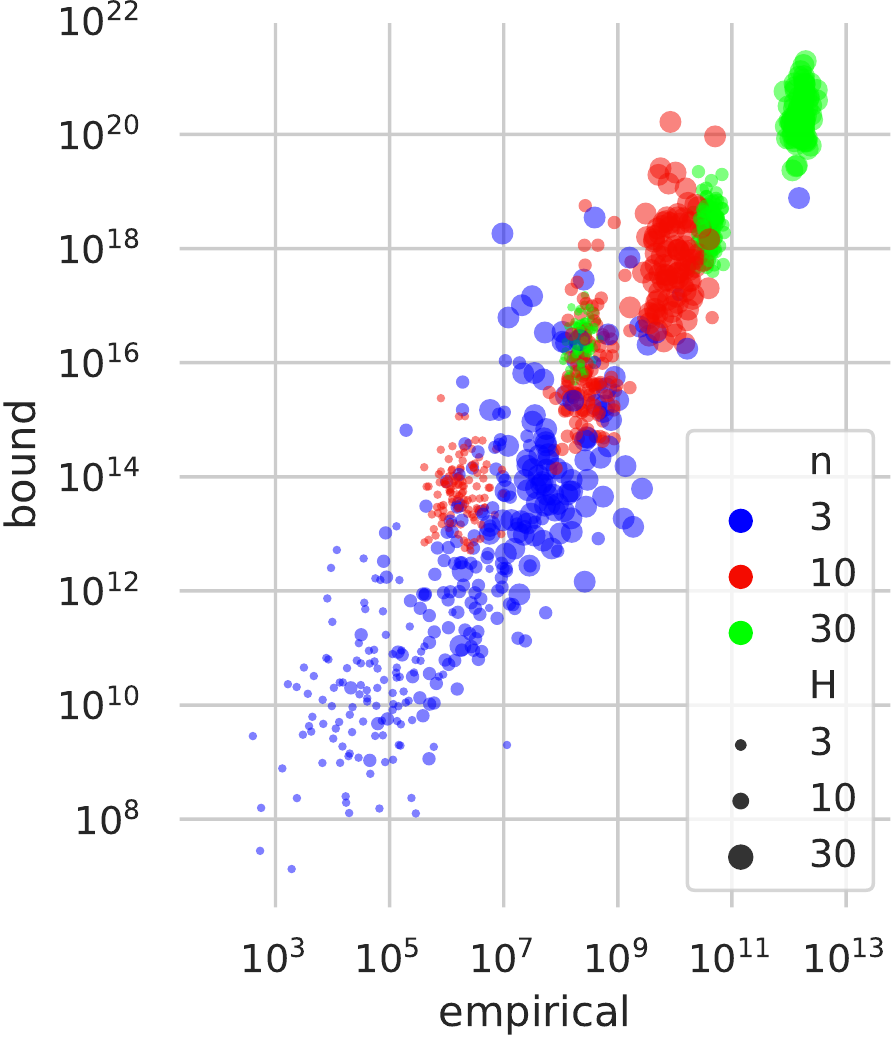}
  \end{center}
  \caption{
  Scatter plot of empirical $\sumvar(\hat g)$ ($x-$axis)
  and upper bound from \Cref{thm:upper} ($y-$axis)
  with varying state dimensionality $n$ and time horizon $H$.
  Each point represents one random LQR problem.
  }
  \label{fig:scatter}
\end{wrapfigure}
\label{sec:rho}
Here we measure the change in variance with respect to
the closed-loop spectral radius $\rho(A + BK)$.
To synthesize controllers $K$ such that $\rho(A + BK)$ obtains a specified value,
we use the pole placement algorithm of \citet{tits-pole-placement}.
A pole placement algorithm $\mathscr{P}$ is a function \[
K = \mathscr{P}(A, B, \lambda_1, \dots, \lambda_n), \quad \lambda_i \in \C,
\]
such that the eigenvalues of $A + BK$ are  $\lambda_1, \dots, \lambda_n$.
We sample a ``prototype'' set of $n$ eigenvalues
with $\lambda_1, \dots, \lambda_{\ceil{n / 2}}$ as complex conjugate pairs
$\lambda_i, \lambda_{i+1} = r e^{\pm i \varphi}$, 
where ${r \sim \Uniform([0, 1])}$ and $\varphi \sim \Uniform([0, \pi))$,
and sample the remaining real $\lambda_i$ from $\Uniform([-1, 1])$.
Then, for each desired $\rho$, we compute
$K_\rho = \mathscr{P}(A, B, \rho \lambda_1, \dots, \rho \lambda_n).$
By rescaling the same set of $\lambda_i$ instead of sampling a new set for each $\rho$,
we avoid confounding effects from changing other properties of $K$.

Results are shown in \Cref{fig:rho}.
Again, we repeat the experiment for different magnitudes of $\Sigma_s$ and $\Sigma_a$.
Unlike the previous two experiments, here we see qualitatively different behavior
between our upper bound and the empirical variance.
The bound begins to increase rapidly near $\rho = 1$,
corresponding to the growth of $1/(1 - \rho)$ in the term $H'$,
but at $\rho = 0.9$ the $H$ term becomes active in $H'$, and the bound suddenly flattens.
In contrast, the empirical variance grows more moderately
and does not explode near the threshold of system instability.
This provides further evidence that the upper bound of \Cref{thm:upper} can be tightened to match
the $\sqrt{H'}$ and $\sqrt{H}$ terms in
the special-case lower bound of \Cref{thm:lower}.

\paragraph{Dimensionality parameters.}
In all of the preceding experiments, we arbitrarily chose the state and action dimensions
${n = 5},\ {m = 3}$
and time horizon $H = 10$.
To visualize the variance for other values of these parameters,
we generate $1000$ random LQR problems with $n$ and $H$ each varying over the set $\{3, 10, 30\}$.
We fix $m = \ceil{n/2}$.
Results are shown in \Cref{fig:scatter}.
The overall positive trend with a slope greater than $1$ shows that
the bound grows superlinearly with respect to the empirical, as expected.
One interesting property is the tighter clustering for large values of $n$.
This may be due to several eigenvalue distribution results in random matrix theory
which state that, as $n \to \infty$, our random LQR problems tend to become similar up to a basis change
\citep{tao-randommatrix}.

\subsection{RL learning curves for varying $\Sigma_a$}
\label{sec:sigmaA-curves}
The results in \Cref{sec:sigmaA} suggest that,
for a fixed $\Sigma_s$, the magnitude of $\Sigma_a$
has a significant effect on $\sumvar(g)$.
This is of practical interest because $\Sigma_a$ is usually under control of the RL practitioner.
It is therefore natural to ask if
the change in variance corresponds to
a change in the rate of convergence of REINFORCE.
We test this empirically by executing REINFORCE in variants of one random LQR problem
with different values of $\Sigma_a$ and $\Sigma_s$.
To avoid a confounding effect from larger $\Sigma_a$
incurring greater penalty from the
$-a_t^T R a_t$ term in $r_t$,
we evaluate the trained policies in a modified version of the problem
where $\Sigma_a = \Sigma_s = \mathbf{0}$.
As discussed in \Cref{sec:lqr}, the optimal $K^\star$ for the stochastic problem
is also optimal for the deterministic problem,
so each problem variant should converge to the same evaluation returns in the limit.

We initialize $K$ by perturbing the elements of the LQR-optimal controller
with i.i.d. Gaussian noise and scaling the perturbation until $\rho(A+BK) \approx 0.98$.
After every $10$ iterations of REINFORCE,
we evaluate the current policy in the noise-free environment.
For each $(\Sigma_a, \Sigma_s)$ pair, we repeat this experiment $10$ times with different random seeds.
The random seed only affects the $\epsilon^a_t, \epsilon^s_t$ and $s_1$ samples.
The aggregate data are shown in \Cref{fig:sigmaA_curves}.
Shaded bands correspond to one standard deviation across the separate runs of REINFORCE. 
The lowercase $\sigma_a, \sigma_s$ refer to scaling factors
applied to the initial samples of $\Sigma_a, \Sigma_s$ in the random LQR problem.

The effect is quite different than one would predict from variance alone.
For all values of $\Sigma_s$ in the experiment,
problems with larger $\Sigma_a$ converge faster---whereas \Cref{fig:sigmaA} would suggest
that the ``optimal'' value of $\Sigma_a$ changes with respect to $\Sigma_s$.
The fact that larger $\Sigma_a$ tends to make REINFORCE converge faster
is not obvious, given the $\Sigma_a^{-1}$ term in $\hat g$ \eqref{eqn:g-hat}.
Also, when $\Sigma_a$ is very small and $\Sigma_s$ is very large,
the algorithm becomes unstable and sees large variations across different random seeds.
For the middle values $\Sigma_a \in \{0.1, 1.0\}$, we observe that larger $\Sigma_s$ causes faster convergence.

\begin{figure}
    \centering
    \includegraphics[width=\columnwidth]{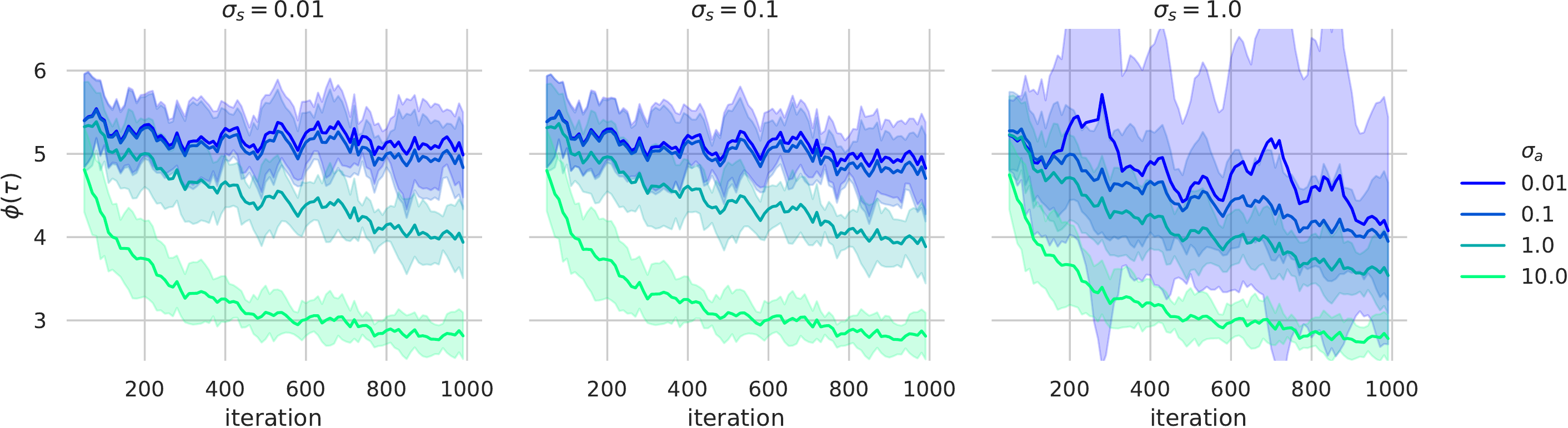}
    \caption{
        Learning curves of REINFORCE for a random LQR problem
        with varying scales of action noise $\sigma_a$ and environment noise $\sigma_s$.
        Larger $\sigma_a$ strictly improves learning, despite larger variance of $\hat g$.
    }
    \label{fig:sigmaA_curves}
\end{figure}

\section{Discussion}
In this work, we derived bounds on the variance of the REINFORCE policy gradient estimator
in the stochastic linear-quadratic control setting.
Our upper bound is fully general, while our lower bound applies to the scalar case
at a stationary point.
The bounds match with respect to all system parameters
except the time horizon $H$ and stability $\rho(A + BK)$.
We compared our bound prediction to the empirical variance in a variety of experimental settings,
finding a close qualitative match in the parameters for which the bounds are tight.

Our experiments in \Cref{sec:sigmaA-curves}
plotting the empirical convergence rate of REINFORCE
suggest that the effect of action noise $\Sigma_a$ on the overall RL performance
is not fully captured by its effect on the variance.
An interesting direction for future work would be to investigate the role of $\Sigma_a$ more closely
and attempt to disentangle its effect on gradient magnitude, variance, exploration, and regularization.
Such an analysis could lead to improved variance reduction methods
or algorithms that manipulate $\Sigma_a$ to speed up the RL optimization.

\subsubsection*{Acknowledgements}
The authors thank
Gaurav S. Sukhatme and Fei Sha
for their discussions regarding this work.
\bibliographystyle{plainnat}
\bibliography{bibliography}{}

\newpage
\appendix

\section{Proof of Theorem~\ref{thm:upper}}
\label{appendix:upper}
In this appendix, we provide the detailed derivation of the upper bound stated in \Cref{thm:upper}.
We define the following notations: let $\delta^s_t$ and $\delta_t^a$ be independent random vectors
that follow $\Norm(\mathbf{0}, I_n)$ and $\Norm(\mathbf{0}, I_m)$ respectively for all $t$.
The $\epsilon^s_t$ and $\epsilon^a_t$ defined in \Cref{eq:lqr-dynamics,eq:lqr-policy}
are thus written as $\epsilon^s_t = \Sigma_s^{\frac{1}{2}}\delta^s_t$
and $\epsilon^a_t = \Sigma_a^{\frac{1}{2}}\delta^a_t$. 
We will use the following steps to upper-bound $\E\left[\tr(\gtg)\right]$:
\begin{enumerate}
\item Bound $\tr(\gtg)$ by $P$,
a polynomial of ${d^s_t \triangleq \|\delta^s_t\|}$ and ${d^a_t \triangleq \|\delta^a_t\|}$
for ${t=1, 2, \ldots, H}$.
We restrict that $P$ should have only nonnegative coefficients.
Since we assume $\delta_t^s$ and $\delta_t^a$ are independent random vectors with standard normal distribution,
$d^s_t$ and $d^a_t$ will be independent random variables following the $\chi(n)$ and $\chi(m)$ distributions respectively.
\item Bound the sum of the (already nonnegative) coefficients of $P$ by substituting all of the $d^s_t,\ d^a_t$ with one. More formally, let 
\begin{align*}
    P(\{d_t^s, d_t^a\}_{t=1}^H) = \sum_{i} C_i \prod_{t=1}^H (d_t^s)^{\alpha_{t,i}} \prod_{t=1}^H (d_t^a)^{\beta_{t,i}}, 
\end{align*}
where $\prod_{t=1}^H (d_t^a)^{\alpha_{t,i}} \prod_{t=1}^H (d_t^s)^{\beta_{t,i}}$ is the $i$-th monomial in $P$, and $C_i$ is its nonegative coefficient. Then we calculate $\sum_i C_i$ by substituting all $d_t^s, d_t^a$ with $1$. We use the notation $\onevars{P}$ to denote this operation (we define the $\onevars{\cdot}$ operator analogously for expressions other than $P$ itself): 
\begin{align*}
    \onevars{P} = P\vert_{d_t^s=1, d_t^a=1, \forall t} = \sum_{i}C_i. 
\end{align*}

\item Bound for the expectation of all monomials in $P$, i.e., find $M$ such that
\begin{align*}
    \E\left[ \prod_{t=1}^H (d_t^s)^{\alpha_{t,i}} \prod_{t=1}^H (d_t^a)^{\beta_{t,i}} \right] \leq M, \quad \forall i. 
\end{align*}
\end{enumerate}

With the three steps above, we can then bound $\E\left[\tr(\gtg)\right]\leq \E[P]\leq \sum_i C_iM = M\onevars{P}$. To calculate $M$, we can use the known formula for the $k$-th moment of a $\chi$ random variable.

Recall that 
\begin{align*}
    \hat{g} = \left(\sum_{t=1}^H \Sigma_a^{-1} \epsilon_t^{a}s_t^\top  \right)\left( \sum_{t=1}^H r_t \right), 
\end{align*}
and thus 
\begin{align*}
    \tr(\gtg) = 
        \underbrace{\tr \left[
		\left( \sumH \Sigma_a^{-1} \epsilon^a_t s_t^\top \right)^\top
		\left( \sumH \Sigma_a^{-1} \epsilon^a_t s_t^\top \right)
	\right]}_{\term{1}}
	\underbrace{
	    \left( \sum_{t=1}^H r_t \right)^2}_{\term{2}}. 
\end{align*}

Thanks to the property $\onevars{P}=\onevars{P_1}\onevars{P}$ for polynomial $P\equiv P_1P$, and $\onevars{P}=\onevars{P_1}+\onevars{P}$ for $P\equiv P_1+P$, we do not need to directly find $P$ and bound $\onevars{P}$. Instead, we can bound $\onevars{P'}$ for some smaller-order component $P'$ of $P$, and then using the above addition/multiplication operations to obtain $\onevars{P}$.  Below in Section~\ref{subsection: bound s_t}, we first bound $\|s_t\|$ by a polynomial $\ceil{s_t}$ of $\{d_\tau^s, d_\tau^a\}$, and then find $\onevars{\ceil{s_t}}$. In Section~\ref{subsectionb: bound term1} and \ref{subsectionb: bound term2}, we further upper bound $\onevars{\term{1}}$ and $\onevars{\term{2}}$ with the help of $\onevars{\ceil{s_t}}$. Then finally we obtain an upper bound for $\E[\tr(\gtg)]$ as $M\onevars{\term{1}}\onevars{\term{2}}$. 

\subsection{Bounding $\|s_t\|$}
\label{subsection: bound s_t}
In this section, we bound $\|s_t\|$ from above.
Although the spectral radius $\rho(A + BK)$ determines the asymptotic stability of the closed-loop system,
it guarantees little about the transient behavior--for example,
for any $x > 1$ and $0 < \epsilon < 1$, the matrix
\[
A = \begin{bmatrix}
\epsilon & x \\
0 & \epsilon
\end{bmatrix}
\]
has the properties $\rho(A) < 1,\ \|A\| > x$.
Therefore, while the state magnitude $\|s_t\|$
is bounded by the operator norm $\|A + BK\|^{t-1}$,
it is too restrictive to require $\|A + BK\| < 1$.
Instead, we will use the following result from the literature:
\begin{lemma}[\citet{trefethen-pseudospectra}]
\label{lemma-resolvent}
Let $A \in \R^{n \times n}$ be a matrix with $\rho(A) < 1$.
Then there exists $\mu > 0$ such that, for all $k \in \N$, 
\begin{equation}
\| A^k \| \leq \mu \rho(A)^k.
\end{equation}
$\mu$ is bounded by the ``resolvent condition'' $\mu \leq 2 e n r(A)$, where $e$ is the exponential constant and 
\begin{equation}
\label{eq:resolvent}
r(A) = \sup_{z \in \C,\ |z| > 1} (|z| - 1) \| (zI - A)^{-1} \|.
\end{equation}
\end{lemma}
The derivation and interpretation of \eqref{eq:resolvent} is a deep subject
related to the matrix \emph{pseudospectrum}, covered extensively by \citet{trefethen-pseudospectra}.
Intuitively, $r(A)$ is large if a small perturbation $\epsilon \in \R^{n \times n}$
would cause $\rho(\epsilon + A) > 1$.

Expanding the state transition function in \Cref{eq:lqr-dynamics} with the linear stochastic policy in \Cref{eq:lqr-policy}, we get 
\begin{align}
    s_t = (A+BK)^{t-1} s_1 +
\sum_{\tau=1}^{t-1} (A+BK)^{t-\tau-1} (\Sigma_s^\half \delta^s_\tau + B \Sigma_a^\half \delta^a_\tau). \label{eqn: s_t expansion}
\end{align}

Recall that $\|\cdot\|$ denotes the $\ell_2$ norm for vectors and the $\ell_2-\ell_2$ operator norm for matrices.
By \Cref{lemma-resolvent}, there exists $\mu$
such that 
$\| (A + BK)^k \| \leq \mu \rho(A)^k$.
Let $\ABKradius = \rho(A + BK)$,
$\sigma_s^2 = \|\Sigma_s\|$,
$\sigma_a^2 = \|\Sigma_a\|$,
and $b = \|B\|$. By repeatedly applying the triangle inequality,
\begin{equation}\begin{split}
\|s_t\|
&= \left\| (A+BK)^{t-1} s_1 +
\sum_{\tau=1}^{t-1} (A+BK)^{t-\tau-1} (\Sigma_s^\half \delta^s_\tau + B \Sigma_a^\half \delta^a_\tau)
\right\| \\
&\leq \left\| (A+BK)^{t-1} s_1 \right\| +
\sum_{\tau=1}^{t-1} \left\| (A+BK)^{t-\tau-1} (\Sigma_s^\half \delta^s_\tau + B \Sigma_a^\half \delta^a_\tau)
\right\| \\
&\leq \mu f^{t-1} \|s_1\| +
\mu \sum_{\tau=1}^{t-1} f^{t-\tau-1} (\sigma_s d^s_\tau + b\sigma_a d^a_\tau).
\label{st_bound}
\end{split}\end{equation}
We denote the final bound in~\eqref{st_bound} as $\ceil{s_t}$.
The bound $\ceil{s_t}$ is linear in the random variables
$\{d^s_t, d^a_t\}_{t=1}^H$
with only positive coefficients.
Furthermore,
\begin{equation}\begin{split}
\onevars{\ceil{s_t}}
&= \mu f^{t-1} \|s_1\| + \mu \sum_{\tau=1}^{t-1} f^{t-\tau-1} (\sigma_s + b\sigma_a ) \\
&= \mu f^{t-1} \|s_1\| + \mu \sum_{\tau=0}^{t-2} f^\tau (\sigma_s + b\sigma_a ).
\end{split}\end{equation}

\subsection{Bounding $\term{1}$}
\label{subsectionb: bound term1}
\begin{lemma}
\[
	\tr \left[
		\left( \sumH \Sigma_a^{-1} \epsilon^a_t s_t^\top \right)^\top
		\left( \sumH \Sigma_a^{-1} \epsilon^a_t s_t^\top \right)
	\right]
	\leq
	\|\Sigma_a^{-1}\| \left( \sumH d^a_t \|s_t\| \right)^2.
\]
\label{lemma-trace-bound}
\end{lemma}
\begin{proof}
Let $\xi_t = \Sigma_a^{-1} \epsilon^a_t$. Then
\begin{equation}\begin{split}
	\tr \left[ \left( \sumH \xi_t s_t^\top  \right)^\top  \left( \sumH \xi_t s_t^\top  \right) \right]
		&= \tr \left[ \sumIJ s_i \xi_i^\top  \xi_j s_j^\top  \right] \\
	&= \sumIJ  \xi_i^\top  \xi_j s_i^\top  s_j \\
	&\leq \sumIJ  |\xi_i^\top  \xi_j| |s_i^\top  s_j| \\
	&= \sumIJ  |{\delta^a_i}^\top  \Sigma_a^{-1} \delta^a_j| |s_i^\top  s_j| \\
	&\leq \sumIJ  \|\Sigma_a^{-1}\| \|\delta^a_i\|  \|\delta^a_j\| \|s_i\| \|s_j\| \\
	&= \|\Sigma_a^{-1}\| \left( \sumH \|\delta^a_t\| \|s_t\| \right)^2\\
	&\leq \|\Sigma_a^{-1}\| \left( \sumH d_t^a \ceil{s_t} \right)^2,
\end{split}\end{equation}
in which we make use of the Cauchy-Schwarz inequality
and the fact that $\Sigma^{\half} \Sigma^{-2} \Sigma^{\half} = \Sigma^{-1}$
for positive semidefinite $\Sigma$.
\end{proof}

\subsection{Bounding $\term{2}$}
\label{subsectionb: bound term2}
We now bound $C \triangleq \sumH -r_t$ from above.
Note that $-r_t \geq 0$, since $Q \succeq 0$ and $R \succ 0$.
Let $q = \|Q\|$, $r = \|R\|$, and $k = \|K\|$.
Then
\begin{equation}\begin{split}
C = \sumH -r_t &= \sumH s_t^\top  Q s_t + a_t^\top  R a_t \\
&\leq \sumH q \|s_t\|^2 + r \|a_t\|^2 
= \sumH q \|s_t\|^2 + r \|K s_t + \Sigma_a^\half \delta^a_t \|^2 \\
&\leq \sumH q \|s_t\|^2 + r(k \|s_t\| + \sigma_a \|\delta^a_t\|)^2 \\
&\leq \sumH q \|s_t\|^2 + 2rk^2 \|s_t\|^2 + 2r\sigma_a^2 \|\delta^a_t\|^2 \\
&\leq (q + 2rk^2) \sumH \ceil{s_t}^2 + 2r\sigma_a^2 \sumH (d^a_t)^2,
\end{split}\end{equation}
where the triangle inequality and the fact $(a + b)^2 \leq 2(a^2 + b^2)$ are used above.
This bound on $C$ is a quadratic polynomial in the $d^s, d^a$.

\subsection{Combining bounds}
Combining Section~\ref{subsectionb: bound term1} and Section~\ref{subsectionb: bound term2}, we have
\begin{equation}\begin{split}
\tr(\gtg) \leq P =
C^2 \|\Sigma_a^{-1}\| \left( \sumH d^a_t \ceil{s_t} \right)^2.
\end{split}\end{equation}
For brevity, let
$\alpha = \|\Sigma_a^{-1}\|$,
$\beta = q + 2rk^2$,
$\gamma = 2r\sigma_a^2$,
$\sigma = \sigma_s + b \sigma_a$,
and $H' \triangleq \min\left\{H ,\frac{1}{1-\ABKradius}\right\}$.
$H'$ reflects the stability of the closed-loop system:
if highly stable ($\ABKradius \ll 1$), we have $H' \ll H$,
but when approaching instability ($\ABKradius \to 1$), $H'$ approaches $H$.

We expand $P$ and substitute $d^s_t = 1,\ d^a_t = 1$ for all $t$ to compute the sum of $P$'s coefficients,
using the notation $\onevars{\cdot}$ for the transformation of replacing all $d$ with $1$. 
We first bound $\onevars{C^2}$:
\begin{equation}\begin{split}
\onevars{C^2} \leq \left( \gamma H + \beta \sumH \onevars{\ceil{s_t}}^2 \right)^2
&\leq \left( \gamma H + \beta \mu^2 \sumH \left[ \ABKradius^{t-1} \|s_1\| + \sigma H' \right ]^2  \right)^2 \\
&\leq \left( \gamma H + 2\beta \mu^2 \sumH \left[ \ABKradius^{2t-2} \|s_1\|^2 + \sigma^2 H'^2 \right ] \right)^2 \\
&\leq \left( \gamma H + 2\beta \mu^2 ( H' \|s_1\|^2 + \sigma^2 HH'^2 ) \right)^2 \\
\end{split}\end{equation}
where the result is obtained by repeatedly applying the fact $(a+b)^2 \leq 2(a^2 + b^2)$.
Next,
\begin{equation}\begin{split}
\onevars{\left( \sumH d^a_t \ceil{s_t} \right)^2}
&\leq \left( \mu \sumH  \ABKradius^{t-1} \|s_1\| + \sigma H' \right)^2 \\
&\leq \mu^2 H'^2(\|s_1\| + \sigma H)^2. \\
\end{split}\end{equation}
Finally,
\begin{equation}\begin{split}
\onevars{P}
&\leq \alpha \mu^2 H'^2 \left( \gamma H + 2\beta \mu^2 ( H' \|s_1\|^2 + \sigma^2 HH'^2 ) \right)^2 (\|s_1\| + \sigma H)^2 \\
&= 4 \|\Sigma_a^{-1}\| \mu^2 H'^2 \left( r\sigma_a^2 H + \mu^2 (q + 2rk^2) ( H' \|s_1\|^2 + \sigma^2 HH'^2 ) \right)^2 (\|s_1\| + \sigma H)^2 \\
& \triangleq \overline{\onevars{P}}.
\end{split}\end{equation}
$\overline{\onevars{P}}$ is an order-8 polynomial in the $d^s, d^a$.
The formula for the 8th moment of a $\chi(n)$ random variable is
\[
\E[X^8] = n(n+2)(n+4)(n+6),
\]
so $\E[\tr(\gtg)] \leq \overline{\onevars{P}} \cdot O(\bar n^4)$, where $\bar n = \max(n,m)$.

(Since we are summing the variances of $O(\bar n^2)$ random variables in $\hat g$,
we would expect scaling of no less than $\bar n^2$ compared to the scalar case.)

\section{$\E [\hat{g}^2]$'s lower bound in the scalar case}
\label{appendix:lower}

In this section, we make the assumption that $m=n=1$, i.e., the states and actions are both scalars.
The matrices $A, B, K, Q, R$ are thus denoted as $a, b, k, q, r$ here
(notice that this notation $r$ is different from the notation $r_t$ for reward).
Other notations follow the definitions in Appendix~\ref{appendix:upper}.
The aim in this appendix is to derive a lower bound for $\E[\hat{g}^2]$ in the special case when $0\leq a+bk< 1$. 

\begin{lemma}
\begin{align}
    \mathbb{E}\left[\left(\sum_{t=1}^H \frac{\epsilon_{t}^{a}s_{t}}{\sigma_a^{2}}  \right)^2\left( \sum_{t=1}^H r_t \right)^2\right] \geq 
    \mathbb{E}\left[\left(\sum_{t=1}^H \frac{\epsilon_{t}^{a}s_{t}}{\sigma_a^{2}}  \right)^2\right]\mathbb{E}\left[\left( \sum_{t=1}^H r_t \right)^2\right]. \label{eqn: lower bound decomposition}
\end{align}
\end{lemma}
\begin{proof}
When $a+bk\geq 0$, all terms have positive coefficients. 
Rename the $2H$ random variables $\{\delta_t^s, \delta_t^a\}_{t=1}^H$ as $x_1, \ldots, x_{2H}$. We can see that a monomial on the right-hand side: 
\begin{align*}
    \E\big[x_1^{\alpha_1}\cdots x_{2H}^{\alpha_{2H}}\big]
    \E\big[x_1^{\beta_1}\cdots x_{2H}^{\beta_{2H}}\big]
\end{align*}
corresponds to the monomial on the left-hand side: 
\begin{align*}
    \E\big[x_1^{\alpha_1+\beta_1}\cdots x_{2H}^{\alpha_{2H}+\beta_{2H}}\big]. 
\end{align*}
The $x_i$ are independent zero-mean normal random variables, so the property $\E[x_i^{\alpha}]\E[x_i^{\beta}]\leq \E[x_i^{\alpha+\beta}]$ holds for any non-negative integers $\alpha, \beta$. Combining with the fact that all coefficients are non-negative shows the lemma.  
\end{proof}

In the following two subsections, we lower bound the two terms on the right-hand side of Eq.~\eqref{eqn: lower bound decomposition} separately. Note that the first term can be simplified as
\begin{align*}
    \left(\sum_{t=1}^H \frac{\epsilon_{t}^{a}s_{t}}{\sigma_a^{2}}  \right)^2 = \left(\sum_{t=1}^H \frac{\delta_{t}^{a}s_{t}}{\sigma_a}  \right)^2 = \frac{1}{\sigma_a^2}\left(\sum_{t=1}^H \delta_{t}^{a}s_{t}  \right)^2. 
\end{align*}

\subsection{Lower bounding $\mathbb{E}\left[\left(\sum_{t=1}^H \delta_{t}^{a}s_{t} \right)^2\right]$}

By the expansion of $s_t$ in Eq.~\eqref{eqn: s_t expansion}, we have
\begin{align*}
 \mathbb{E}\left[\left(\sum_{t=1}^H \delta_{t}^{a}s_{t} \right)^2\right] 
 &=\mathbb{E}\left[\left(\sum_{t=1}^H \delta_{t}^{a} \left((a+bk)^{t-1}s_1 + L_t\right)\right)^2\right] \\
 &\geq \mathbb{E}\left[\left(\sum_{t=1}^H \delta_t^a(a+bk)^{t-1}s_1\right)^2\right] + \mathbb{E}\left[\left(\sum_{t=1}^H \delta_t^a L_t \right)^2\right],  
\end{align*}
where $L_t\triangleq \sum_{\tau=1}^{t-1}(a+bk)^{t-1-\tau}(\sigma_s\delta_\tau^s+b\sigma_a\delta_\tau^a)$. 
The first term is equal to $\sum_{t=1}^H (a+bk)^{2t-2}s_1^2=\frac{1-(a+bk)^{2H}}{1-(a+bk)^2}s_1^2$. The second term can be further written as
\begingroup
\allowdisplaybreaks
\begin{align*}
  &\mathbb{E}\left[\left(\sum_{t=1}^H \delta_t^a\sum_{\tau=1}^{t-1}(a+bk)^{t-1-\tau}(\sigma_s\delta_\tau^s+b\sigma_a\delta_\tau^a) \right)^2\right] \\  
  & = \mathbb{E}\left[\sum_{t=1}^H (\delta_t^{a})^2\left(\sum_{\tau=1}^{t-1}(a+bk)^{t-1-\tau}(\sigma_s\delta_\tau^s+b\sigma_a\delta_\tau^a) \right)^2\right] \\
  & = \mathbb{E}\left[\sum_{t=1}^H \left(\sum_{\tau=1}^{t-1}(a+bk)^{t-1-\tau}(\sigma_s\delta_\tau^s+b\sigma_a\delta_\tau^a) \right)^2\right] \\
  & = \mathbb{E}\left[\sum_{t=1}^H \sum_{\tau=1}^{t-1}(a+bk)^{2t-2-2\tau}(\sigma_s^2 + b^2\sigma_a^2) \right] \\
  & = (\sigma_s^2 + b^2\sigma_a^2) \mathbb{E}\left[ \sum_{t=1}^H \frac{1-(a+bk)^{2t-2}}{1-(a+bk)^2}  \right] \\
  & = (\sigma_s^2 + b^2\sigma_a^2) \left( \frac{H}{1-(a+bk)^2} - \frac{1-(a+bk)^{2H}}{(1-(a+bk)^2)^2} \right). 
\end{align*}
\endgroup
In the above several equalities, we use the independence among $\delta_t^a,\delta_t^s$.
Combining two terms, we get
\begin{align*}
    \mathbb{E}\left[\left(\sum_{t=1}^H \delta_{t}^{a}s_{t} \right)^2\right]&\geq \frac{s_1^2 + H(\sigma_s^2+b^2\sigma_a^2)}{1-(a+bk)^2} - \frac{(a+bk)^{2H}}{1-(a+bk)^2}s_1^2 - (\sigma_s^2+b^2\sigma_a^2)\frac{1-(a+bk)^{2H}}{(1-(a+bk)^2)^2}\\
    &= \left( \frac{1-(a+bk)^{2H}}{1-(a+bk)^2}s_1^2 +\left(H-\frac{1-(a+bk)^{2H}}{1-(a+bk)^2}\right)\frac{\sigma_s^2+b^2\sigma_a^2}{1-(a+bk)^2}\right)\\
    &\approx \begin{cases}
                 \frac{s_1^2+H(\sigma_s^2+b^2\sigma_a^2)}{1-(a+bk)^2}  & \text{when\ } H\gg \frac{1}{1-(a+bk)^2} \\
                 H\left(s_1^2+H(\sigma_s^2+b^2\sigma_a^2)\right) &\text{when\ } H\ll \frac{1}{1-(a+bk)^2} 
            \end{cases}\\
    & \approx \min\left\{H, \frac{1}{1-(a+bk)^2}\right\}\left( s_1^2+H(\sigma_s^2+b^2\sigma_a^2) \right)\\
    & \approx H'\left( s_1^2+H(\sigma_s^2+b^2\sigma_a^2) \right). 
\end{align*}

\subsection{Lower bounding $\mathbb{E}\left[\left( \sum_{t=1}^H r_t \right)^2\right]$}
We first lower bound this term by
\begin{align*}
    \E\left[\left(\sum_{t=1}^H r_t\right)^2\right]
    &\geq 
    \E\left[\sum_{t=1}^H r_t\right]^2 
    = \E\left[\sum_{t=1}^H qs_t^2 + r(ks_t+\sigma_a\delta_t^a)^2 \right]^2 \\
    &= \E\left[ \sum_{t=1}^H (q+rk^2)s_t^2 + 2rk s_t \sigma_a\delta_t^a + r\sigma_a^2\delta_t^{a2} \right]^2 \\
    &= \left((q+rk^2)\E\left[ \sum_{t=1}^H s_t^2 \right] + Hr\sigma_a^2\right)^2
\end{align*}
\begin{align*}
    \mathbb{E}\left[ \sum_{t=1}^H s_t^2 \right] 
    &= \mathbb{E}\left[ \sum_{t=1}^H \left((a+bk)^{t-1}s_1+L_t\right)^2 \right] \\
    & = \mathbb{E}\left[\sum_{t=1}^H \left((a+bk)^{2t-2} s_1^2 + 2(a+bk)^{t-1}s_1 L_t + L_t^2\right)\right] \\
    & = \mathbb{E}\left[\sum_{t=1}^H (a+bk)^{2t-2} s_1^2 + \sum_{t=1}^{H}\left(\sum_{\tau=1}^{t-1} (a+bk)^{t-1-\tau}(\sigma_s \delta_t^s + b\sigma_a \delta_t^a) \right)^2\right] \tag{the middle term $\E[(a+bk)^{t-1}s_1L_t]$ is zero because $L_t$ is a sum of zero-mean RVs}   \\
    & = \frac{1-(a+bk)^{2H}}{1-(a+bk)^2}s_1^2 + \mathbb{E}\left[ \sum_{t=1}^H \sum_{\tau=1}^{t-1}(a+bk)^{2t-2-2\tau}(\sigma_2^2+b^2\sigma_a^2) \right] \\
    & = \frac{1-(a+bk)^{2H}}{1-(a+bk)^2}s_1^2 + (\sigma_s^2 + b^2\sigma_a^2) \left( \frac{H}{1-(a+bk)^2} - \frac{1-(a+bk)^{2H}}{(1-(a+bk)^2)^2} \right)\\
    & \approx \min\left\{H, \frac{1}{1-(a+bk)^2}\right\}\left( s_1^2+H(\sigma_s^2+b^2\sigma_a^2) \right)\\
    & \approx H'\left( s_1^2+H(\sigma_s^2+b^2\sigma_a^2) \right).
\end{align*}
The second-to-last approximation is obtained similarly as in the previous subsection. 
\subsection{Combining}
Combining the results in previous two subsections, we get the final lower bound on $\E[\hat g^2]$:
\begin{align*}
    \mathbb{E}[\hat{g}^2] &= \mathbb{E}\left[\left(\sum_{t=1}^H \frac{\delta_{t}^{a}s_{t}}{\sigma_a}  \right)^2\left( \sum_{t=1}^H r_t \right)^2\right] \geq 
    \mathbb{E}\left[\left(\sum_{t=1}^H \frac{\delta_{t}^{a}s_{t}}{\sigma_a}  \right)^2\right]\mathbb{E}\left[\left( \sum_{t=1}^H r_t \right)^2\right] \\
            &\geq \Omega(c_1^2c_2^2), 
            \end{align*}
where (recall $\sigma \triangleq \sigma_s + b\sigma_a$)
\begin{align*}
    c_1 &= \frac{1}{\sigma_a}\left(|s_1| + \sigma \sqrt{H}\right)\sqrt{H'}, \\
    c_2 &= r\sigma_a^2 H + (q+rk^2)(s_1^2+\sigma^2 H)H'.  
\end{align*}

\end{document}